\documentclass[11pt]{article}
\usepackage[letterpaper,margin=1in]{geometry}
\usepackage{amsmath}
\usepackage{amsfonts}
\usepackage{amsthm}
\usepackage{thmtools}
\usepackage{thm-restate}
\usepackage{xcolor}
\definecolor{darkgreen}{rgb}{0,0.5,0}
\definecolor{darkblue}{rgb}{0,0,0.7}
\usepackage[hidelinks,colorlinks=true,
            linkcolor=darkblue,
            citecolor=darkgreen,
            urlcolor=darkblue]{hyperref}
\usepackage[nameinlink,capitalise]{cleveref}
\crefname{algocf}{algorithm}{algorithms}
\Crefname{algocf}{Algorithm}{Algorithms}
\usepackage[autostyle=true]{csquotes}

\usepackage[T1]{fontenc}
\usepackage{bbm}
\usepackage{appendix}

\usepackage[ruled,vlined]{algorithm2e}
\usepackage{xcolor}
\usepackage{enumitem}
\usepackage{subcaption}
\usepackage{tikz}
\usepackage{nicefrac}
\usepackage{graphicx}
\usepackage{wrapfig}
\usepackage{natbib}

\theoremstyle{plain}
\newtheorem{theorem}{Theorem}[section]
\newtheorem*{theorem*}{Theorem}
\newtheorem{lemma}[theorem]{Lemma}
\newtheorem{proposition}[theorem]{Proposition}

\theoremstyle{definition}
\newtheorem{definition}[theorem]{Definition}

\theoremstyle{remark}

\DeclareMathOperator*{\argmax}{arg\,max}

\newcommand{\R}{\mathbb{R}}

\newcommand{\ip}[1]{\left\langle #1 \right\rangle}

\newcommand{\E}[2]{\mathbb{E}_{#1}\left[#2\right]}
\newcommand{\V}[2]{{\mathbb{V}}_{#1}\left[#2\right]}

\renewcommand{\P}[2]{{\mathbb{P}}^{#1}\left[#2\right]}
\renewcommand{\exp}[1]{{\textup{exp}}\left(#1\right)}
\newcommand{\ind}[1]{\mathbbm{I}\left\{#1\right\}}

\newcommand{\ALG}{\textup{\textsc{Alg}}}

\newcommand{\Reg}{\textup{\textsc{Reg}}}

\title{Online Convex Optimization with Sublinear Noisy Probes}
\date{}

\author{Simone Di Gregorio\thanks{Sapienza University of Rome, Italy (simone.digregorio@uniroma1.it, leonardi@diag.uniroma1.it)} \and Anupam Gupta\thanks{New York University, USA (anupam.g@nyu.edu)}
\and Stefano Leonardi\footnotemark[1]
\and Matteo Russo\thanks{EPFL, Switzerland (matteo.russo@epfl.ch)}}
\begin{document}

\maketitle

% The abstract
\begin{abstract}
We study Online Convex Optimization (OCO) over a convex set $K\subseteq \R^d$, where in each round $t$ the learner selects $x_t\in K$ and then observes a convex loss $f_t:K\to[0,1]$, with the goal of minimizing regret to the best fixed decision in hindsight. We introduce a unified probing model that generalizes two recent lines of work: sublinear \emph{best-expert} queries in the experts setting \citep{RCCFHKL24}, and pairwise (comparison-based) feedback available every round in OCO \citep{BhaskaraGIKM23}. In our framework, the learner has a budget of $k\le T$ \emph{pairwise probes}; on a probed round it may query two points and learn which one has smaller loss.

Our main result shows that even a \emph{sublinear and noisy} probe budget can provably improve worst-case regret in the full feedback OCO regime. With $k$ $\delta$-noisy pairwise probes, we obtain
\[
\Reg_T \le O\!\left(\min\left\{\sqrt{dT\ln T},\; \frac{dT\ln T}{k|1-2\delta|}\right\}\right),
\]
which is tight (up to logarithmic factors in $T$) across $T$, $k$ and $\delta$. Specifically regarding the noise parameter $\delta \in [0,1]$, 
the regret guarantee smoothly degrades as the oracle response approaches a coin flip, i.e., $\delta$ is close to $\nicefrac{1}{2}$. When applying the same techniques to a finite $K$ for the prediction with $d$ experts setting, the resulting rates are instead completely tight in all parameters, including $d$. Our analysis gives a streamlined treatment of pairwise probing in OCO by quantifying the benefit of probing via a variance reduction effect, combined with a second-order (variance-based) analysis of Continuous Exponential Weights \citep{Bubeck11,RooijEGK14}.
\end{abstract}

%The title page has no number and only contains title and abstract
%\clearpage
\pagenumbering{arabic}
\newif\ifcoltversion
\coltversionfalse
\section{Introduction}\label{sec:intro}
Online convex optimization (OCO) is a general framework for sequential
decision-making in which, over a horizon of $T$ rounds, a learner
repeatedly selects a decision $x_t$ from a convex set $K\subseteq\R^d$,
incurs an adversarially chosen convex loss $f_t(x_t)$, and receives
feedback at the end of each round \citep{Shalev-Shwartz12, Hazan16, orabona}.  The
learner's performance is measured by \emph{regret}, defined as the gap
between its cumulative loss and that of the best fixed comparator
$x^\star\in K$ in hindsight:
\[
\Reg_T = \sum_{t=1}^T f_t(x_t) \;-\; \inf_{x\in K}\sum_{t=1}^T f_t(x).
\]
In the \emph{full-information} variant, the entire loss function $f_t$
(or equivalently a subgradient oracle at every point) is revealed after
each round; classical methods such as Online Gradient Descent and
Follow-the-Regularized-Leader achieve $\Theta(\sqrt{T})$ regret (up to
problem-dependent geometry factors). In the bandit version, only $f_t(x_t)$ is revealed for the played point $x_t$.

The OCO framework subsumes Online linear optimization (OLO), where losses are linear and expressed as $f_t(x)=f_t^\top x$ for vector $f_t \in \R^d$. As a further special case, we have the classical online learning with finitely many
actions/experts: taking $\Delta^{d-1}$ (the $d-1$-dimensional
simplex) and linear losses, the regret
against the best vertex recovers the usual notion of regret in
the experts problem. In this discrete specialization, full
feedback reveals all coordinates of $f_t$, while bandit feedback
reveals only $f_t(x_t)$ for the played expert $x_t$.

Motivated by the recent surge of algorithms augmented by external
predictions, hints, or side-information \citep{MitzenmacherV20}, we
study a model in which the learner may occasionally obtain limited
\emph{advance} information about the losses \emph{before} committing
to its decision on that round. Most relevant to our setting are
\citet{BhaskaraGIKM23} and \citet{RCCFHKL24}, who show that even very
restricted query access can dramatically reduce regret.  In
\citet{BhaskaraGIKM23}, the learner is allowed \emph{pairwise
  comparison probes}: at the beginning of each round $t$, before
choosing the decision $x_t$, the learner may choose two candidate
decisions and learn which one will incur smaller instantaneous loss in
this round, without revealing the loss values themselves. They obtain
time-independent regret bounds for online linear optimization and
improved rates for OCO under additional curvature assumptions. In the
discrete bandits setting, they show that pairwise comparisons can
yield regret of order $O(d\ln T)$ in the stochastic
case. Complementarily, \citet{RCCFHKL24} allow \emph{best-expert
  probes} (which they call queries) in the experts special case only,
which reveal the identity of the minimizing expert in the current
round $t$ (before the decision $x_t$ is played); however, only
$k\le T$ such queries can be made. %They show that in full information
In this setting, they show a 
regret bound of $O\!\left(\min\left\{\sqrt{T\ln d}, \nicefrac{T\ln d}{k}\right\}\right)$.\\

These results raise the following natural question:
\begin{center}
    \emph{If the learner is able to issue only a sublinear number $k \ll T$ of (possibly noisy) \underline{pairwise} probes, how much can the OCO regret be improved with respect to the base $\sqrt{T}$ rate?}
\end{center}
To address this question, we first formalize an OCO-with-probing model
that unifies both discrete and continuous settings, and then state our
results.

\subsection{Our Model}\label{sec:model}

We present a probing-augmented online optimization protocol. The OCO and prediction with expert advice
settings arise as two specializations. 

\subsubsection{Online convex optimization with noisy probing}
\label{subsec:model_oco}

Let $K\subset \mathbb{R}^d$ be either a compact convex set or a finite set. The interaction proceeds for $T$ rounds against an
\emph{oblivious adversary} who pre-commits to a sequence of loss
functions bounded in $[0,1]$: $\{f_t:K\to[0,1]\}_{t=1}^T$.
At each round $t\in\{1,\dots,T\}$, the learner picks
$x_t\in K$, incurs loss $f_t(x_t)$, and then receives feedback $f_t(x)$ for all $x \in K$. Moreover, for appropriate functions $f$ and $g$ we define $\ip{f,g} = \int_K f(x)g(x)\mu(dx)$, where the base measure $\mu$ is explicit or clear from the context.

\paragraph{Probing budget.}
In addition to the feedback, the learner may issue \emph{probes}
in at most $k$ rounds. Let $Q\subseteq\{1,\dots,T\}$ denote the set of probed rounds, decided (possibly randomly and adaptively) by the learner, with $|Q|\le k$. If $t\in Q$, before selecting $x_t$, the learner may
specify a (possibly randomized) finite set of $C$ candidate points
$
P_t=\{y_{t,1},\dots,y_{t,C}\}\subseteq K
$,
and the probe returns the identity of the candidate with minimal
instantaneous loss in $P_t$, possibly corrupted by noise. In
this paper we focus on the least powerful nontrivial probes, namely
pairwise comparisons with $C=2$.

To model noise, we assume that the
probe returns the candidate with the \emph{maximal} (instead of
minimal) loss with probability $\delta$, but our results clearly apply to less challenging
noise: for example, the oracle could return, with probability $\delta$, any of the two probed points, which is better than returning the $\arg\max$, and our regret bounds still hold.
Moreover, any probe model with general $C > 2$ can be reduced to the pairwise case by constructing $P_t$ as multiset made up of $y_{t,1}$ repeated $C-1$ times and $y_{t,2}$ once: our upper bounds thus extend for general $C > 2$.
% \ag{Who decides what $Q$ is? Is $Q$ decided up-front, or adaptively?
%   Right now this is unclear.}

\begin{definition}[Noisy comparison probe]\label{def:probes_oco}
Fix $\delta\in[0,1]$ and an integer $C\geq 1$.
A \emph{$\delta$-noisy (comparison) probe} at round $t$ is a (possibly randomized)
set $P_t=\{y_{t,1},\dots,y_{t,C}\}\subseteq K$ together with an outcome:
\[
\hat{y}_t \;=\;
\begin{cases}
y_t^\star \in \arg\min_{y\in P_t} f_t(y), & \text{with probability } 1-\delta,\\
y_t^-\in \argmax_{y\in P_t} f_t(y), & \text{with probability } \delta.
\end{cases}
\]
Whenever the minimizer or maximizer is not unique, ties are broken uniformly at random. When $\delta=0$, we call the probe \emph{noiseless}. Moreover, a probe is \emph{global} when $P_t = K$ and $\delta = 0$, that is the probe outcome is $y^\star_t \in \arg\min_{y\in K} f_t(y)$ with probability $1$, assuming the minimum is achieved.\footnote{If the minimum is not achieved, one may model the notion of global probe by letting it return any point in a sublevel set characterized by a function value arbitrary close to the minimum.}
\end{definition}

\paragraph{Protocol.}
Formally, first an
oblivious adversary pre-commits to a sequence of convex loss
functions $f_1, \ldots, f_T$ bounded in $[0,1]$, and then, in each
round $t=1,\dots,T$:
\begin{enumerate}
\item %\ag{Should we say, learner decides if $t \in Q$. If so, then
  % learner does blah?}\mr{Yes, it is much clearer, updated.}
  The learner decides if $t\in Q$: if so, it chooses $P_t\subseteq K$
 and observes $\hat{y}_t\in P_t$.
\item The learner chooses $x_t\in K$ (possibly as a randomized function of $\hat{y}_t$ and
  the past).
\item The learner incurs loss $f_t(x_t)$ and observes full feedback $\{f_t(x)\}_{x\in K}$.
\end{enumerate}

\paragraph{Regret.}
For any point $x\in K$, define the cumulative loss
$L_T(x)=\sum_{t=1}^T f_t(x)$. Similarly, let $L_T(\ALG) = \sum_{t=1}^T f_t(x_t)$ be the cumulative loss of a learner $\ALG$ whose decisions over rounds are $x_1, \ldots, x_T$, each belonging to $K$. In turn, these points are sampled from the probability measures induced by the densities $p_1, \ldots, p_T$ set by the learner $\ALG$ across rounds. In the following, for any time $t>0$, we abuse notation and write $x\sim p_t$ for $x$ being sampled from the probability measure induced by $p_t$ and the base measure $\mu$. Moreover, let $v_t = \V{x \sim p_t}{f_t(x)}$ be the variance of the function at time $t$ according to $p_t$, and let $V_t = \sum_{\tau \le t} v_\tau$ be the cumulative variance up until time $t$.  
The (worst-case) expected regret
of an algorithm $\ALG$ is
\[
\E{}{\Reg_T(\ALG)}=
\sup_{\boldsymbol{f}}
\left\{
\E{}{\sum_{t=1}^T f_t(x_t)}
\;-\;
\inf_{x\in K}\sum_{t=1}^T f_t(x)
\right\},
\]
where the expectation is with respect to the learner's internal
randomness and the probe noise, and the supremum ranges over all
obliviously chosen sequences $\boldsymbol{f}=\{f_t\}_{t=1}^T$ of admissible
losses bounded in $[0,1]$; for online convex optimization, we restrict to convex losses. Our goal is sublinear regret, i.e., $\E{}{\Reg_T(\ALG)}=o(T)$, while leveraging at most
$k$ noisy comparison probes.

% \paragraph{On the power of the adversary.} The above protocol dictates that the adversary is the first to move, and the learner moves second. This could be extended to an \emph{adaptive adversary} that, at time $t$, sees all the decisions taken by the learner until time $t-1$, and (possibly) selects $f_t$ based on that. We stress that it is crucial for the adversary not to be able to couple losses construction with probing decisions taken by the learner. Indeed, if the adversary could, then (even with noiseless probes) she would set $f_t\equiv 0$ for all rounds where the learner probes and the best attainable regret would be $\sqrt{T-k}$ as opposed to the $\nicefrac{T}{k}$ rates we obtain in this work (see \Cref{sec:discussion} for a discussion).  

\subsection{Our Results and Technical Challenges}

Our main contribution is to show that even a \emph{sublinear} number of pairwise comparison probes can substantially reduce regret. We unify the two probing paradigms of \citet{BhaskaraGIKM23,RCCFHKL24} into a single framework in which the learner is granted a budget of $k\le T$ \emph{pairwise} (possibly noisy) probes over a horizon of $T$ rounds, and we analyze the resulting regret tradeoffs for full feedback OCO and prediction with experts.

\paragraph{Regret guarantees.}
With $k$ pairwise $\delta$-noisy probes, we prove that one can achieve
\begin{equation}
  \Reg_T \le O\!\left(\min\left\{\sqrt{dT\ln T},\; \frac{dT\ln T}{k\,|1-2\delta|}\right\}\right),
  \label{eq:1}
\end{equation}
for general convex losses bounded in $[0,1]$ (see
\Cref{thm:noise-oco}). In the special case of experts (i.e., linear
losses over the vertices of the simplex $\Delta^{d-1}$), we can
replace a
factor of $d\ln
T$ by $\ln d$ for both terms of the $\min$ in \eqref{eq:1} to get an improved
bound of 
\begin{equation}
  \Reg_T \le O\!\left(\min\left\{\sqrt{T\ln d},\; \frac{T\ln d}{k\,|1-2\delta|}\right\}\right).
  \label{eq:1b}
\end{equation}
A key takeaway is that the regret bounds of \eqref{eq:1b} match those
obtained in \citet{RCCFHKL24} under \emph{best-expert} probes, where
in each queried round the learner effectively ``looks ahead'' and
learns the globally optimal action $x_t^\star$. In the special case of
experts \emph{only}, they leverage a crucial property of best-expert
probes: on a probed round, the learner can guarantee the minimum
possible loss, since it knows the expert with lowest loss at that
round. This yields a non-positive (and potentially negative)
contribution to regret during probing rounds. Combined with a standard
Hedge bound on the remaining rounds, this leads to worst-case regret
$O\left(\tfrac{T\ln d}{k}\right)$ once $k\ge\Omega(\sqrt{T\ln
  d})$. 
  
Making this argument with pairwise probes runs into a basic
obstacle even just for the special case of experts: to benefit, one
must compare against (or identify) a near-best expert with
sufficiently large probability. One approach to doing this is adapting the arguments from \citet{RCCFHKL24} and mix
a uniformly random probe with probes from the distribution of the
Hedge algorithm, to avoid hurting the baseline performance. In \Cref{app:strawman-experts}, we carry over this analysis, attaining suboptimal rates even in the noiseless case. In contrast, our algorithms manage to get the (optimal) bounds in \Cref{eq:1b} using only (noisy) pairwise probes, moreover getting the rate in \Cref{eq:1} for general bounded convex losses over an
arbitrary convex body $K$! 

\paragraph{Our approach and the main technical insight.}
Our main theorem (\Cref{thm:noise-oco}) 
simultaneously handles (i) general convex losses over general convex bodies, (ii) sublinear probe budgets, and (iii) noisy probes, recovering essentially tight (in terms of $T, k$ and $\delta$) bounds on the regret rate. The algorithm is a direct and adaptive probe-augmented version of \emph{continuous} exponential weights: it maintains a density $p_t$ over $K$ and samples points accordingly (e.g., \citet{Bubeck11}). This route avoids differential privacy machinery and dispenses with curvature/gradient assumptions \citep{BhaskaraGIKM23}.

For noiseless probing, the key insight is that a pairwise comparison drawn from $p_t$ yields an advantage that can be measured by the \emph{variance} of the losses under $p_t$. Concretely, on a probing round $t$, comparing two i.i.d.\ samples from $p_t$ and playing the better one decreases the expected loss by $\V{x\sim p_t}{f_t(x)}$, the variance of the function according to the density $p_t$:
\[
    \E{p_t}{f_t(x_t) \mid t \in Q} \le \E{p_t}{f_t(x_t)} - \V{p_t}{f_t(x_t)}.
\]
We then couple this with a variance-based analysis of continuous exponential weights (following the less standard ``second-order'' style guarantees of \citet{RooijEGK14}), which relates the baseline regret to the same variance term. This alignment is what enables the $\nicefrac{T}{k}$-type improvement without an additional factor in $d$, nor making assumptions about gradient norms or curvature.

\paragraph{Noisy probes and learning when to trust them.}
Noisy probes introduce a further challenge: the noise level $\delta$ is unknown, so the learner must also infer whether to trust the probe outcome. We cast this as a lightweight meta-learning layer that, on probed rounds, decides whether to \emph{follow} the probe's suggestion or \emph{invert} it. Instantiating this meta-problem with an adaptive two-action learner allows us to recover the bound in \eqref{eq:1} with the correct dependence on $|1-2\delta|$.

\paragraph{Lower bounds.}
Finally, we show that the rates in~(\ref{eq:1b}) are tight in all parameters for prediction with expert advice, and those
in~(\ref{eq:1}) are tight in $T, k$ and $\delta$ (up to logarithmic factors in $T$) for
general convex losses over arbitrary convex sets; we give these
results in \Cref{sec:lb}.

\subsection{On the Power of the Adversary and the Definition of Noisy Probes}
% \label{sec:discussion}
The protocol above requires the adversary to fix the loss $f_t$ before the learner decides whether to probe at time $t$. It is crucial that the adversarial loss construction is not coupled with the probing decision. If the adversary could anticipate a probe, she could set $f_t \equiv 0$ for that round, rendering the probe useless. This would degrade the optimal regret to $\Theta(\sqrt{T-k})$, preventing the faster $\tilde \Theta(\nicefrac{T}{k})$ rates we achieve. This constraint distinguishes our work from the ``adversarially noisy'' probes considered in the literature \citep{BhaskaraC0P20, BhaskaraC0P21, BhaskaraC0P23, BhaskaraCKP21, BhaskaraGIKM23}. In those settings, the adversary effectively knows when probing occurs. 
For instance, \citet{BhaskaraGIKM23} derive a lower bound of $\Omega(\sqrt{T-k})$ by constructing an instance where the adversary sets losses to $(0,0)$ specifically during noiseless probes (and $(0,1)$ vs. $(1,0)$ each with probability $\nicefrac{1}{2}$ otherwise). This contradicts our results because their model assumes the adversary has access to the learner's internal randomization (knowing \emph{when} a probe happens), whereas our model does not.

%%%%%%%%%%%%%%%%%%%%%%%%%%%%%%%%%%%%%%%%%%%%%%%%%%%%%%%%%%%%

\section{Related Work}\label{sec:rel-work}

The model of learning-augmented algorithms \citep{MitzenmacherV20} has been most recently considered in the context of Online Linear Optimization (OLO) over a convex polytope \citep{Shalev-Shwartz12}. Here, the learner has access to a query mechanism that returns vectors correlated with the true loss vectors. When the optimization domain is the $d$-dimensional unit sphere and the algorithm has full feedback, this additional information enables a reduction in regret from $O(\sqrt{T})$ to $O(\ln T)$, even if the learner only receives such hints on $O(\sqrt{T})$ rounds \citep{BhaskaraC0P20, BhaskaraC0P21, BhaskaraCKP21}. Interestingly, this is not possible with bandit feedback, and a $\Omega(\sqrt{T})$ regret is unavoidable even if the hint vectors are well correlated with the true ones \citep{BhaskaraC0P23}. The same work shows that active queries to specific points at which to evaluate the loss function are enough to guarantee regret rates of the form $O(d^{3/2}\ln T)$, where $d$ is the dimension of the ambient space. Their techniques do not directly apply to our setting, since their probes are defined differently from ours and those from \citep{BhaskaraGIKM23}. 
 Moreover, in subsequent work, \cite{BacchiocchiCMS26} show that under bandit feedback, even in the experts setting and with access to best-arm probes, the worst-case regret remains $\Omega(\sqrt{T - k})$.

Regarding pairwise probes for OCO in particular, the closest work to ours is
  \citet{BhaskaraGIKM23}, who obtain a $\Theta(Hd+Gd^{3/2})$ regret
  bound in the full feedback setting when probing is available every
  round ($k=T$), under uniform bounds $G,H$ on gradient norms and
  curvature. Their analysis relies on a differential-privacy-based
  reduction to Be-The-Regularized-Leader (BTRL), 
  and the
  dependence on curvature and gradient bounds appears intrinsic to
  their arguments, which we avoid using our techniques.

Techniques developed in this work, as well as in prior ones, aim to adapt to structure in the loss sequence during non-probing rounds, leading to regret guarantees that depend on the "easiness" of the instance. These forms of easiness include stochastic losses or small effective loss ranges. In full-information settings, such structure can be leveraged by algorithms like Hedge to achieve regret depending on the easiness of the instance, without degrading worst-case performance \citep{Cesa-BianchiMS07, GaillardSE14, KoolenE15, LuoS15}. Another relevant direction is the predictable sequences framework, which focuses on scenarios where the loss sequence follows a regular pattern or is correlated with past observations \citep{RakhlinS13}. Algorithms in this line adjust their learning dynamics to exploit such predictability and achieve improved regret when it exists \citep{SteinhardtL14a, WeiL18}. Limited supervision has also been studied through the abstention learning framework, where learners may choose not to predict and instead defer to an oracle, thereby trading off prediction effort with access to external feedback \citep{LiLWS11, SayediZB10, ZhangC16, CortesDGMY18, NeuZ20, GangradeKCS21}.

Our work is also connected to stochastic probing, where the learner acts under uncertainty and can selectively probe the environment to obtain partial feedback \citep{GuptaN13, GuptaNS17, Singla18}. This model has applications in problems like Pandora's Box \citep{Weitzman79, BeyhaghiK19, BeyhaghiC23}, online matching \citep{Singla18}, and submodular optimization \citep{PattonR023}. A central challenge in these settings is deciding when and what to probe to optimize long-term performance under uncertainty, and recent algorithms address this by carefully balancing exploration with exploitation \citep{AgarwalGN24}.
\section{Main Algorithm and Regret Guarantees}
\label{sec:main}

In this section, we show how a sublinear number of (noisy) probes enables a significant drop in regret rates. To do so, we recall a generic algorithmic template based on Continuous Exponential Weights with different priors $p_1$ over the domain $K$. This template has been used by, e.g., \citet{RooijEGK14} for prediction with expert advice, and by \citet{Bubeck11} for OCO: we augment it with a \textsc{Meta-Learner} and $k$ pairwise comparison probes issued uniformly at random throughout the time horizon (see below). Different instantiations of \Cref{alg:cew-probe-noisy} guarantee regret bounds for general convex classes and stronger bounds for the experts setting (see Section 3.5). 

\paragraph{Meta-Learner.} To handle the noisy case with unknown $\delta$, the idea is to run a meta-learner on the \textit{``follow the oracle''} and the \textit{``invert the oracle''} actions, succinctly denoted as $F$ and $I$ respectively. Essentially, on probing rounds, we do what the \textsc{Meta-Learner} says: if the routine \textsc{Meta-Learner} returns $F$, the algorithm follows the probe and plays $\hat y_t$, otherwise it plays the other probed point. After receiving feedback, the \textsc{Meta-Learner} is updated based on the value of $f_t$ in the probed points.

Here, the \textsc{Meta-Learner} is simply AdaHedge (see Theorem 5 in \citet{Cesa-BianchiMS07} and Theorem 6 in \citet{RooijEGK14}) run on the two actions $F,I$ and for rounds $Q \subseteq [T]$ only (see pseudocode in \Cref{alg:cew-probe-noisy} for convenience). AdaHedge is a standard Hedge routine with a time-adaptive learning rate, which allows getting variance-based regret bounds. The losses for $F$ and $I$ are simply the values of $f_t$ obtained by either following or inverting the oracle.
\ifcoltversion
\begin{algorithm2e}
    \SetAlgoNoEnd
    \caption{Probe-Augmented Continuous Exponential Weights with Noisy Probes}
    \label{alg:cew-probe-noisy}
    \SetKwInOut{Input}{Input}
    
    \Input{Sequence of measurable functions $f_t$, domain $K\subseteq \mathbb{R}^d$, prior density $p_1$ with respect to base measure $\mu$ over $K$, probe budget $k \le T$, parameter $\Lambda$}
    
    Let $Z_1 = \int_K p_1(x) \mu(dx) = 1$, $V_0 = 0$. Uniformly sample $k$ out of $T$ rounds, obtaining $Q$
    
    \For{$t=1,\dots,T$}{
        \eIf{$t\in Q$}{
            Probe $2$ points $y_{t,1},y_{t,2} \sim p_t$ independently and observe \(\hat y_t\)
            
            \eIf{\textsc{Meta-Learner}$(\cdot) = F$}{Play $\hat y_t$}{Play $y_{t, 1}$ if $y_{t, 1}\neq \hat y_t$ else play $y_{t, 2}$}
        }{
            Play $x_t \sim p_t$\;
        }
        
        Observe loss function $f_t(x)$ for all $x \in K$\;
        
        Set $V_t = V_{t-1}+ v_t$ and  $\eta_{t+1} = \min\left(\tfrac{1}{2}, \sqrt{\tfrac{\Lambda}{V_t + 1}}\right)$
        
        Update density for all $x \in K$ as
        \[
            p_{t+1}(x) = \frac{p_1(x) \cdot \exp{-\eta_{t+1} L_t(x)}}{Z_{t+1}}
            \quad\text{where}\quad
            Z_{t+1} = \int_K p_1(x) \cdot \exp{-\eta_{t+1} L_t(x)} \mu(dx)
        \]

        Update \textsc{Meta-Learner} using $\{f_t(y_{t, 1}), f_t(y_{t, 2})\}$ \textbf{if} $t\in Q$
    }
\end{algorithm2e}
\else
\begin{algorithm}
    \SetAlgoNoEnd
    \caption{Probe-Augmented Continuous Exponential Weights with Noisy Probes}
    \label{alg:cew-probe-noisy}
    \SetKwInOut{Input}{Input}
    
    \Input{Sequence of measurable functions $f_t$, domain $K\subseteq \mathbb{R}^d$, prior density $p_1$ with respect to base measure $\mu$ over $K$, probe budget $k \le T$, parameter $\Lambda$}
    
    Let $Z_1 = \int_K p_1(x) \mu(dx) = 1$, $V_0 = 0$. Uniformly sample $k$ out of $T$ rounds, obtaining $Q$
    
    \For{$t=1,\dots,T$}{
        \eIf{$t\in Q$}{
            Probe $2$ points $y_{t,1},y_{t,2} \sim p_t$ independently and observe \(\hat y_t\)
            
            \eIf{\textsc{Meta-Learner}$(\cdot) = F$}{Play $\hat y_t$}{Play $y_{t, 1}$ if $y_{t, 1}\neq \hat y_t$ else play $y_{t, 2}$}
        }{
            Play $x_t \sim p_t$\;
        }
        
        Observe loss function $f_t(x)$ for all $x \in K$\;
        
        Set $V_t = V_{t-1}+ v_t$ and  $\eta_{t+1} = \min\left(\tfrac{1}{2}, \sqrt{\tfrac{\Lambda}{V_t + 1}}\right)$
        
        Update density for all $x \in K$ as
        \[
            p_{t+1}(x) = \frac{p_1(x) \cdot \exp{-\eta_{t+1} L_t(x)}}{Z_{t+1}}
            \quad\text{where}\quad
            Z_{t+1} = \int_K p_1(x) \cdot \exp{-\eta_{t+1} L_t(x)} \mu(dx)
        \]

        Update \textsc{Meta-Learner} using $\{f_t(y_{t, 1}), f_t(y_{t, 2})\}$ \textbf{if} $t\in Q$
    }
\end{algorithm}
\fi

\begin{theorem}
\label{thm:noise-oco}
    Consider the problem of online convex optimization with full feedback and $k$ pairwise comparison probes: $K\subset \R^d$ is convex and compact, with non-empty interior.
    Then, for an unknown noise parameter $\delta \in [0,1]$ and for any sequence of convex functions $\{f_t\}_{t=1}^T$ bounded in $[0,1]$, when instantiated with $p_1$ being the uniform density over $K$ and $\Lambda = d\ln T$, \Cref{alg:cew-probe-noisy} has regret
            \[
                \E{}{\Reg_T(\ALG_k)} \leq O\left(\min\left(\sqrt{dT\ln T}, \frac{dT\ln T}{k|1-2\delta|}\right)\right).
            \]
\end{theorem}

In order to prove the above theorem, let us first write the following decomposition that holds for any fixed density $q$ over $K$ and any fixed point $x^\star \in K$:
\begin{align}\label{eq:regret-decomposition}
    \E{}{\sum_{t=1}^T f_t(x_t)} - \sum_{t=1}^T f_t(x^\star) = \underbrace{\E{}{\sum_{t=1}^T f_t(x_t)} - \sum_{t=1}^T \ip{q, f_t}}_{\text{(A)}} + \underbrace{\sum_{t=1}^T \ip{q, f_t} - \sum_{t=1}^T f_t(x^\star)}_{\text{(B)}}.
\end{align}
We bound term (A) and (B) separately; convexity of $f_t$'s is used only for bounding (B).

\subsection{Loss during non-probing rounds}

In this subsection, we compare the performance of \Cref{alg:cew-probe-noisy} without considering probes against any distribution $q$ over the domain $K$. We assume that $q \ll p_1$, i.e. $q$ is absolutely continuous w.r.t. $p_1$, where by this we mean that the property is satisfied by the two underlying measures. The analysis is similar in spirit to the one present in \cite{Cesa-BianchiMS07,RooijEGK14}. Before proceeding, we bound the average loss using the log-partition function and a variance term.
\begin{restatable}{lemma}{lemvarianceonestepbound}\label{lem::variance-one-step-bound}
    It holds that
    \[
        \ip{p_t, f_t} \le \frac{1}{\eta_t}\ln Z_t - \frac{1}{\eta_{t+1}}\ln Z_{t+1} + \frac{\eta_t}{2(1-\eta_t)} \cdot v_t
    \]
\end{restatable}
    \begin{proof}
   For this proof, we recall the following bound (e.g. \citet[Proposition 2.10]{Wainwright19}):
   \begin{proposition}[Bernstein Subexponential Tail Bound]
    \label{prop:bernstein_bound}
    For a bounded random variable $X$ with $|X - \E{}{X}| \leq 1$, it holds that:
    \begin{align*}
        \ln \E{}{\exp{-\eta X}} \leq -\eta \cdot \E{}{X} + \frac{\eta^2}{2(1-|\eta|)} \cdot \V{}{X} \qquad \forall\, \eta: |\eta| < 1.\label{fct:bernstein}
    \end{align*}
    \end{proposition}
   
    To use the above, we define an intermediate normalizer, which uses the same learning rate $\eta_t$ but after observing $f_t$; conventionally, in what follows, we let $\eta_1 = \nicefrac{1}{2}$: 
    \begin{align*}
        \widetilde Z_{t+1} = \int_K p_1(x) \cdot \exp{-\eta_t L_t(x)} \mu(dx) &= Z_t \cdot \int_K p_t(x) \cdot \exp{-\eta_t f_t(x)} \mu(dx) \\
        &= Z_t \cdot \E{x_t \sim p_t}{\exp{-\eta_t f_t(x_t)}}.
    \end{align*} 
    Applying Proposition \ref{prop:bernstein_bound} to the bounded random variable $f_t(x_t)$ with density $p_t$, we then obtain:
    \begin{align*}
        \ln \left(\frac{\widetilde Z_{t+1}}{Z_t}\right) = \ln \E{x_t \sim p_t}{\exp{-\eta_t f_t(x_t)}}&\leq -\eta_t \cdot \E{x_t \sim p_t}{f_t(x_t)} + \frac{\eta_t^2}{2(1-\eta_t)} \cdot \V{x_t \sim p_t}{f_t(x_t)} \\
        &= -\eta_t \cdot \ip{p_t, f_t} + \frac{\eta_t^2}{2(1-\eta_t)} \cdot v_t.
    \end{align*}
    Hence, using that $\eta_t\leq 1$ and rearranging:
    \[
        \ip{p_t, f_t} \le \frac{1}{\eta_t}\ln \left(\frac{Z_t}{\widetilde Z_{t+1}}\right) + \frac{\eta_t}{2(1-\eta_t)} \cdot v_t.
    \]
    Let us now write
    \begin{align*}
        Z_{t+1} = \int_K p_1(x) \cdot \exp{-\eta_{t+1} L_t(x)} \mu(dx) &= \int_K p_1(x) \cdot \exp{-\eta_t L_t(x)}^{\nicefrac{\eta_{t+1}}{\eta_t}} \mu(dx) \\
        &\le \left(\int_K p_1(x) \cdot \exp{-\eta_t L_t(x)} \mu(dx)\right)   ^{\nicefrac{\eta_{t+1}}{\eta_t}} \\
        & = (\widetilde Z_{t+1})^{\nicefrac{\eta_{t+1}}{\eta_t}},
    \end{align*}
    where the inequality follows by Jensen's inequality since the function $y \mapsto y^{\nicefrac{\eta_{t+1}}{\eta_t}}$ is concave in $y$ as $\eta_{t+1} \le \eta_t$. Combining the above two displays, we have:
    \begin{align*}
        \ip{p_t, f_t} \le \frac{1}{\eta_t}\ln \left(\frac{Z_t}{\widetilde Z_{t+1}}\right) + \frac{\eta_t}{2(1-\eta_t)} \cdot v_t \le \frac{1}{\eta_t}\ln Z_t - \frac{1}{\eta_{t+1}}\ln Z_{t+1} + \frac{\eta_t}{2(1-\eta_t)} \cdot v_t,
    \end{align*}
    as desired.
\end{proof}

\begin{lemma}\label{lem:noprobe-oco}
    Let $q$ be an absolutely continuous density with respect to $p_1$. If $\Lambda \ge \max(1, \textup{KL}(q||p_1))$, it holds that \Cref{alg:cew-probe-noisy} without probes, referred to as $\ALG_0$, satisfies:
    \begin{align*}
        \sum_{t=1}^T \ip{p_t-q, f_t} \le 2\Lambda + 5\sqrt{\Lambda(1 + V_T)}.
    \end{align*}
\end{lemma}
\begin{proof}
    First, we use \Cref{lem::variance-one-step-bound} and obtain via a telescopic sum throughout $t$:
    \[
        \sum_{t=1}^T \ip{p_t,f_t} \le \frac{1}{\eta_1}\ln Z_1 - \frac{1}{\eta_{T+1}}\ln Z_{T+1} + \sum_{t=1}^T \frac{\eta_t}{2(1-\eta_t)} \cdot v_t = - \frac{1}{\eta_{T+1}}\ln Z_{T+1} + \sum_{t=1}^T \frac{\eta_t}{2(1-\eta_t)} \cdot v_t,
    \]
    where we have observed that $Z_1 = \int_K p_1(x) \mu(dx) = 1$ by definition, so that $\ln Z_1 = 0$. We next bound the two terms separately. Now, for all $q \ll p_1$, we have the following: 
 
    \begin{align*}
        Z_{T+1} = \int_{x \in K} p_1(x) \exp{-\eta_{T+1}L_T(x)} \mu(dx) &\geq \int_{x\in K:\,q(x)>0} q(x)\frac{p_1(x)}{q(x)} \exp{-\eta_{T+1}L_T(x)} \mu(dx) \\
        &= \E{x \sim q}{\exp{-\eta_{T+1}L_T(x) - \ln\left(\frac{q(x)}{p_1(x)}\right)}} \\
        &\geq \exp{\E{x \sim q}{-\eta_{T+1}L_T(x) - \ln\left(\frac{q(x)}{p_1(x)}\right)}},
    \end{align*}
    where the first and third inequalities hold because the integrand is positive and because of Jensen’s inequality, respectively. The expectation is well-defined since $q\ll p_1$. We thus have:
    \[
        - \frac{1}{\eta_{T+1}}\ln Z_{T+1} \le \frac{1}{\eta_{T+1}}\left(\E{x \sim q}{\eta_{T+1}L_T(x)} + \textsc{KL}(q||p_1)\right) = \sum_{t=1}^T \ip{q, f_t} + \frac{\textsc{KL}(q||p_1)}{\eta_{T+1}}.
    \]
    Second, choosing $\eta_t = \min\left(\tfrac{1}{2}, \sqrt{\tfrac{\Lambda}{1 + V_{t-1}}}\right)$, we have that $\tfrac{\eta_t}{2(1-\eta_t)} \le \eta_t$ and also (e.g., Lemma 14 in \citet{GaillardSE14}):
    {\allowdisplaybreaks
    \begin{align*}
        \sum_{t=1}^T \frac{\eta_t}{2(1-\eta_t)} \cdot v_t &\le 1+\sum_{t=2}^T \eta_t v_t \le 1+ \sqrt{\Lambda} \cdot \sum_{t=2}^T \frac{v_t}{\sqrt{1 + V_{t-1}}} \\
        &= \,1+ \sqrt{\Lambda} \cdot \sum_{t=2}^T \frac{v_t}{\sqrt{1 + V_t}} + \sqrt{\Lambda}\cdot \sum_{t=2}^T v_t \cdot \left(\frac{1}{\sqrt{1 + V_{t-1}}} - \frac{1}{\sqrt{1 + V_t}}\right)\\
        &\le 1+\sqrt{\Lambda} \cdot \sum_{t=2}^T \frac{v_t}{\sqrt{1 + V_t}} + \sqrt{\Lambda} \cdot \sum_{t=2}^T \left(\frac{1}{\sqrt{1 + V_{t-1}}} - \frac{1}{\sqrt{1 + V_t}}\right)\\
        &\le 1 + \sqrt{\Lambda} \left(1+ \sum_{t=2}^T \frac{v_t}{\sqrt{1 + V_t}}\right) \le 1 + \sqrt{\Lambda}\left(1+ \int_{V_1}^{V_T} \frac{dx}{\sqrt{1+x}}\right)\\
        &\leq 1 + 3\sqrt{\Lambda(1 + V_T)} .
    \end{align*}
    }
    Combining the derivations above, we obtain, using our assumption on $\Lambda \ge \max(1, \textsc{KL}(q||p_1))$:
    \[
        \sum_{t=1}^T \ip{p_t-q, f_t} \le \frac{\textsc{KL}(q||p_1)}{\eta_{T+1}} +1+3\sqrt{\Lambda(1 + V_T)} \le 2\Lambda + 5\sqrt{\Lambda(1 + V_T)},
        % 4\left(1+\sqrt{\Lambda(1 + V_T)}\right),
    \]
    as desired.
\end{proof}

\subsection{Loss during probing rounds}
\label{subsec:loss_probing}
We decompose the loss for probing rounds in two components: (i) the loss assuming $\delta$ is known (\Cref{lem:probe-noisy-1})---corresponding to the loss suffered by the superior action between $F$ and $I$; and (ii) the excess regret incurred relative to that superior action (\Cref{lem:probe-noisy-2}).

In the subsequent lemmas, we denote the instantaneous losses of actions $F$ and $I$ by $\ell_t(F)$ and $\ell_t(I)$, and their cumulative losses over the rounds in $Q$ by $L_Q(F)$ and $L_Q(I)$, respectively. These losses are random variables contingent on the sampling of the outer algorithm and the Bernoulli noise variables governing the probe responses (\Cref{def:probes_oco}). To avoid conditioning on $Q$ before it is needed, we extend the definition of $\ell_t(F)$ and $\ell_t(I)$ to rounds $t\notin Q$ by considering the necessary Bernoulli noise and pairs $y_{t, 1}, y_{t, 2}$, which are completely excluded from the learning protocol.

\begin{restatable}{lemma}{lemprobenoisy}
\label{lem:probe-noisy-1}
    The following holds:
    \begin{align*}
        \E{p_t}{\ell_t(F)} &\le \ip{p_t,f_t} - (1-2\delta) v_t \quad \text{if}\ \delta \leq \nicefrac{1}{2} \\
        \E{p_t}{\ell_t(I)} &\le \ip{p_t,f_t} + (1-2\delta) v_t\quad \text{if}\ \delta > \nicefrac{1}{2}.
    \end{align*}
    Therefore, accounting for the choice of $Q \subseteq [T]$, the expected cumulative loss accumulated by the better of the two actions $F$ and $I$ is:
    \[
        \min\{\E{}{L_Q(F)}, \E{}{L_Q(I)}\} \le \frac{k}{T} \cdot \left(\sum_{t = 1}^T \ip{p_t,f_t} - |1-2\delta| \cdot V_T\right).
    \]
\end{restatable}
\begin{proof}
    First, assume $\delta\leq \nicefrac{1}{2}$. By definition, we have that the action \textit{``follow the oracle''} incurs loss:
    \begin{align*}
       \E{p_t}{\ell_t(F)} &= \int_{K \times K} p_t(u)p_t(y) \cdot (1-\delta) \cdot \min(f_t(u), f_t(y)) \\
       &\qquad \qquad + \delta \cdot \max(f_t(u), f_t(y)) \,\mu(du)\,\mu(dy)\\
        &= \int_{K \times K} p_t(u)p_t(y) \cdot \frac{f_t(u)+f_t(y)-(1-2\delta) \cdot |f_t(u)-f_t(y)|}{2} \,\mu(du)\,\mu(dy)\\
        &\le \ip{p_t, f_t} - (1-2\delta) \cdot \int_{K \times K} p_t(u)p_t(y) \cdot \frac{(f_t(u)-f_t(y))^2}{2} \,\mu(du)\,\mu(dy)\\
        &= \ip{p_t, f_t} - (1-2\delta) \cdot \left(\int_K p_t(u) f^2_t(u) \,\mu(du) - \ip{p_t, f_t}^2\right)\\
        &= \ip{p_t, f_t} - (1-2\delta) v_t.
    \end{align*}
    Summing over $t \in Q$, we have $\E{}{L_Q(F)\, |\, Q} \le \sum_{t \in Q} \ip{p_t,f_t} - (1-2\delta) \cdot \sum_{t \in Q} v_t$. Similarly, if $\delta > \nicefrac{1}{2}$, we have, via a derivation almost identical to the one above:
    $
    \E{}{L_Q(I) \, |\, Q}\leq \sum_{t\in Q}\ip{p_t, f_t}+(1-2\delta)\sum_{t\in Q}v_t
    $, thus getting a bound for the action \textit{``invert the oracle''}. 
    Hence, taking another expectation over $Q$ and combining the two cases depending on the sign of $1-2\delta$, we get:
    \[
        \min\{\E{}{L_Q(F)}, \E{}{L_Q(I)}\} \le \frac{k}{T} \cdot \left(\sum_{t = 1}^T \ip{p_t,f_t} - |1-2\delta| \cdot \sum_{t = 1}^T v_t\right),
    \]
    which concludes the proof.
\end{proof}

\begin{lemma}\label{lem:probe-noisy-2}
    The $\textup{\textsc{Meta-Learner}}$ has expected cumulative loss bounded by
    \[
        \E{}{L_Q(\textup{\textsc{Meta-Learner}})} \le \min\{\E{}{L_Q(F)}, \E{}{L_Q(I)}\} + \sqrt{\frac{2k}{T}V_T} + 3.
    \]
\end{lemma}
\begin{proof}
    First, recall that the \textsc{Meta-Learner} runs Hedge with adaptive learning rate (i.e., AdaHedge from \citet{RooijEGK14} run on the two actions $F,I$ and for rounds $Q \subseteq [T]$ only). For any $t\in Q$, let $\tilde p_{t, F}, \tilde p_{t, I}$ be the probabilities maintained by the \textsc{Meta-Learner} over $F,I$. Similarly, let $\tilde v_t$ be the instantaneous variance induced by $\tilde p_{t, F}, \tilde p_{t, I}$ of the \textsc{Meta-Learner}, i.e., $\tilde v_t = \tilde p_{t, F} \tilde p_{t, I} (\ell_t(F) - \ell_t(I))^2$, and let $\tilde V_Q = \sum_{t \in Q} \tilde v_t$. Observe that
    \[
        \tilde v_t = \tilde p_{t, F} \tilde p_{t, I} (\ell_t(F) - \ell_t(I))^2 \le \frac{1}{4}(\ell_t(F) - \ell_t(I))^2.
    \]
    Also note that $\lvert \ell_t(F) - \ell_t(I)\rvert = \lvert f_t(y_{t,1}) - f_t(y_{t,2})\rvert = \Delta_t$, which means that $\tilde V_Q \le \tfrac{1}{4}\sum_{t \in Q} \Delta^2_t$. In particular, since $y_{t, 1}$ and $y_{t, 2}$ are i.i.d. and taking the difference outputs a centered random variable, $\E{}{\sum_{t\in Q}\Delta_t^2\, |\, Q} = \sum_{t=1}^T 2v_t \ind{t\in Q}$, where $v_t$ is the variance at $t$ of the original learning task. 
    
    In addition, Theorem 6 in \citep{RooijEGK14} gives a regret of the \textsc{Meta-Learner} (AdaHedge) with respect to the better of $F,I$ of at most, letting $\mathcal G$ be the sigma-algebra generated by $Q$, the noise variables and the sampling from $p_t$:
    \[
    \E{}{
       L_Q(\textup{\textsc{Meta-Learner}}) -  \min\{L_Q(F), L_Q(I)\}\, |\, \mathcal G} \le 2\sqrt{\tilde V_Q\ln 2} + \frac{4\ln 2}{3} + 2 \le 2\sqrt{\tilde V_Q} + 3.
    \]
    Combining the above and taking expectations, we obtain
    \begin{align*}
        \E{}{\E{}{L_Q(\textup{\textsc{Meta-Learner}})\,|\, \mathcal G}} &\le \E{}{\min\{L_Q(F), L_Q(I)\}} + 2\E{}{\sqrt{\tilde V_Q}} + 3 \\
        &\le \min\{\E{}{L_Q(F)}, \E{}{L_Q(I)}\} + 2\sqrt{\E{}{\tilde V_Q}} + 3 \\
        &\le \min\{\E{}{L_Q(F)}, \E{}{L_Q(I)}\} + \sqrt{\E{}{\E{}{\sum_{t \in Q}\Delta_t^2\, \Bigg|\, Q}}} + 3 \\
        &= \min\{\E{}{L_Q(F)}, \E{}{L_Q(I)}\} + \sqrt{\sum_{t=1}^T 2v_t\P{}{t\in Q}} + 3\\
        &= \min\{\E{}{L_Q(F)}, \E{}{L_Q(I)}\} + \sqrt{\frac{2k}{T}V_T} + 3,
    \end{align*}
    where the second inequality holds by using Jensen's inequality twice, and the second-to-last equality by using standard properties of conditional expectations.
\end{proof}

\subsection{Proving the regret bound}
We can finally prove the claimed regret bound from \Cref{thm:noise-oco}, and we start by bounding term (A) in \Cref{eq:regret-decomposition} using the results from the previous two subsections. 

\begin{lemma}\label{lem:noisy-oco-combined}
    Let $q$ be a probability density function on the domain $K$, such that $q \ll p_1$.
    Then, for an unknown noise parameter $\delta \geq 0$ and for any sequence of measurable functions $f_t$ bounded in $[0,1]$, \Cref{alg:cew-probe-noisy} guarantees, provided $\Lambda \geq \max(1, \textup{KL}(q||p_1))$:
    \begin{align*}
        \E{}{L_T(\ALG_k)} - \sum_{t=1}^T \ip{q, f_t} \le \min\left(12\sqrt{\Lambda T} , \frac{25\Lambda T}{k|1-2\delta|}\right)+4.
    \end{align*}
\end{lemma}
\begin{proof}
     Let us decompose the loss suffered by $\ALG_k$ into probing and non-probing rounds:
     \[
        \E{}{L_T(\ALG_k)} = \E{}{L_{T \setminus Q}(\ALG_k)} + \E{}{L_Q(\ALG_k)}.
     \]
     In particular, note that the loss suffered by $\ALG_k$ during the probing rounds is equal to that of the \textsc{Meta-Learner}. Since we have the same type of update to $p_t$ in both probing and non-probing rounds, we have that $\E{}{L_{T \setminus Q}(\ALG_k)\,|\, Q} = \sum_{t=1}^T\ind{t\notin Q}\E{}{f_t(x_t)}$, and since $Q$ is uniform: 
    \begin{align*}
        \E{}{L_{T \setminus Q}(\ALG_k)}&= \left(1-\frac{k}{T}\right) \sum_{t=1}^T\ip{p_t, f_t}.
    \end{align*}
    By \Cref{lem:probe-noisy-1} and \Cref{lem:probe-noisy-2}, it also holds that:
     \begin{align*}
         \E{}{L_Q(\ALG_k)} &= \E{}{L_Q(\textup{\textsc{Meta-Learner}})} \le \min\{\E{}{L_Q(F)}, \E{}{L_Q(I)}\} + \sqrt{\frac{2k}{T}V_T} + 3\\
        &\le  \frac{k}{T} \cdot \left(\sum_{t = 1}^T \ip{p_t,f_t} - |1-2\delta| \cdot V_T\right) + \sqrt{\frac{2k}{T}V_T} + 3.
     \end{align*}
    
    Therefore, combining the two above displays, we get
    \begin{align*}
        \E{}{L_T(\ALG_k)} &\le \sum_{t=1}^T\ip{p_t, f_t} + \sqrt{\frac{2k}{T}V_T} + 3 - \frac{k|1-2\delta|}{T} \cdot V_T\\
        &\le \sum_{t=1}^T \ip{q, f_t} + 2\Lambda + \underbrace{5\sqrt{\Lambda(1 + V_T)} + \sqrt{\frac{2k}{T}V_T}}_{\le 7\sqrt{\Lambda(1 + V_T)}} + 3 - \frac{k|1-2\delta|}{T} \cdot V_T\\
        &\le \sum_{t=1}^T \ip{q, f_t} + \min\left(10\sqrt{\Lambda T} , \frac{23\Lambda T}{k|1-2\delta|}\right)+4 + 2\Lambda,
    \end{align*}
    where the second inequality holds by \Cref{lem:noprobe-oco}, while the third because $V_T\leq T$ and the maximizing $V_T$ is $V_T = \tfrac{49T^2\Lambda}{4k^2(1-2\delta)^2} -1$. To get the final bound for the first argument of the minimum, notice that $\Lambda\leq \sqrt{\Lambda T}$ if $\Lambda \leq T$, while if $\Lambda > T$, then the loss is trivially bounded by $T<\sqrt{\Lambda T}$. The bound for the second argument follows since $\nicefrac{T}{k\lvert1-2\delta\rvert}\geq 1$, with the convention of setting it to $\infty$ when $\delta = \nicefrac{1}{2}$.
\end{proof}

We now prove \Cref{thm:noise-oco} by removing the dependency of the previous analysis on $q$.
\ifcoltversion
\begin{proof}[of \Cref{thm:noise-oco}]
\else
\begin{proof}[Proof of \Cref{thm:noise-oco}]
\fi
     First, define $S(x^\star, r)=(1-r)x^\star + rK$, for a generic $x^\star\in K$. Now fix any $t\in [T]$ and let $q$ be the uniform density over $S(x^\star, r)$, i.e. $q(x) = \frac{1}{\text{Vol}(S(x^\star\!, r))}\ind{x\in S(x^\star, r)}$: the base measure in \Cref{alg:cew-probe-noisy} is thus in this case the Lebesgue measure. By definition of Minkowski sum, for any $y\in S(x^\star,r)$ there exists $x\in K$ such that $y=(1-r)x^\star+rx$.
    By convexity of $f_t$,
    \[
    f_t(y)\le (1-r)f_t(x^\star)+r f_t(x)\le (1-r)f_t(x^\star)+r,
    \]
    since $f_t(x)\le 1$. Rearranging then gives $f_t(y)-f_t(x^\star)\le r$ for all $y\in S(x^\star,r)$. This addresses the (B) term in \Cref{eq:regret-decomposition}.
    
    To address the (A) term, notice that $q\ll p_1$ and $\textsc{KL}(q || p_1) = d\ln(\nicefrac{1}{r}) = \Lambda$, since:
    \[
        \textsc{KL}(q || p_1) = \ln\left(\frac{\text{vol}(K)}{\text{vol}(S(x^\star, r))}\right) = \ln\left(\frac{\text{vol}(K)}{r^d\text{vol}(K)}\right) = d\ln\left(\frac{1}{r}\right).
    \]
    Plugging the equality in the bound from \Cref{lem:noisy-oco-combined}, then yields:
    \begin{align*}
        \E{}{L_T(\ALG_k)} - \sum_{t=1}^T f_t(x^\star) &\le \min\left(12\sqrt{dT\ln(\nicefrac{1}{r})}, \frac{25dT\ln(\nicefrac{1}{r})}{k|1-2\delta|}\right) + 4 + rT,
    \end{align*}
    and setting $r = \nicefrac{1}{T}$ gives the claimed bound.
\end{proof}

\subsection{Prediction with Expert Advice}
% \snote{We need some remark here stressing that we use this notation just because we want the result to be compatible with }
In the prediction with expert advice setting,
\Cref{eq:regret-decomposition} and \Cref{lem:noisy-oco-combined} imply a regret bound, 
% for the predictions with experts setting
with the proof relying on the fact that the convexity of domains and functions was only used when proving \Cref{thm:noise-oco}. We thus have to appropriately switch to the discrete setting: the domain is finite, where the experts can be seen as the $d$ vertices of $\Delta^{d-1} = \{x \in \R^d \mid x_i\geq 0 \,\forall\,i ,\sum_{i=1}^d x_i = 1\}$, the $(d-1)$-dimensional simplex. Recall that its set of vertices is $\{e_1, \ldots, e_d\}$.\footnote{We stress that we are using this notation for finite sets of points to match the notation used in the lower bounds in \Cref{sec:lb}.}

\begin{theorem}
\label{thm:experts_regret}
    Consider the problem of prediction with expert advice with full feedback and $k$ pairwise comparison probes, with the domain $K = \{e_1, \ldots, e_d\}$.
    Then, for an unknown noise parameter $\delta \ge 0$ and for any sequence of functions $\{f_t\}_{t=1}^T$ in $[0, 1]$, when instantiated with $p_1$ being the mass function of the uniform distribution over $\{e_1, \ldots, e_d\}$ and $\Lambda = \ln d$, \Cref{alg:cew-probe-noisy} has regret
            \[
                \E{}{\Reg_T(\ALG_k)} \leq O\left(\min\left(\sqrt{T\ln d}, \frac{T\ln d}{k|1-2\delta|}\right)\right).
            \]
\end{theorem}
\begin{proof}
For this result, the base measure of \Cref{alg:cew-probe-noisy} is the counting measure over $\{e_1, \ldots, e_d\}$, while $p_1$ is the probability mass function of the uniform distribution over the same set. Consider the mass function $q(x) = \ind{x = x^\star}$, thus corresponding to a Dirac on $x^\star$, here the minimizing expert, so that $\E{x \sim q}{f_t(x)} = f_t(x^\star)$, and term (B) in \Cref{eq:regret-decomposition} is null. Clearly, $q\ll p_1$. For term (A), we have:
    \[
        \textsc{KL}(q || p_1) = \sum_{x \in \{e_1, \ldots, e_d\}} q(x) \cdot \ln\left(\frac{q(x)}{p_1(x)}\right) = \ln d.
    \]
    Therefore, since we set $\Lambda = \ln d$, 
    \Cref{lem:noisy-oco-combined} gives:
    \begin{align*}
        \E{}{L_T(\ALG_k)} - \sum_{t=1}^T f_t(x^\star) \le \min\left(12\sqrt{T\ln d}, \frac{25T\ln d}{k|1-2\delta|}\right)+4,
    \end{align*}
    which was to be shown.
    \end{proof}
\section{Concluding Remarks}\label{sec:conclusion}

In this work, we analyze Online Convex Optimization (OCO) with a sublinear budget of pairwise probes. This generalizes \cite{BhaskaraGIKM23}, which requires pairwise probes at \emph{every} step, as well as \cite{RCCFHKL24}, which focuses on the \emph{stronger} best-expert probes. Our main finding is that pairwise probes are as informative as best-expert probes and that our algorithms are robust to noise. In \Cref{thm:noise-oco}, we establish the following regret rates, which are tight in $T, k$ and $\delta$ up to logarithmic factors in $T$, and for experts completely tight also in $d$:
\[
    \E{}{\Reg_T(\ALG_k)} \leq O\left(\min\left(\sqrt{dT\ln T}, \frac{dT\ln T}{k|1-2\delta|}\right)\right).
\]
Among the remaining open questions, one concerns understanding the benefit of probes in the special case of exp-concave functions, as opposed to general convex ones. Another relates to the trade-off between computational efficiency and improved regret: our main CEW algorithm requires computing an integral over a convex set, which is hard in general. It would therefore be interesting to develop efficient algorithms that achieve comparable guarantees.
\section*{Acknowledgments}

A.G.\ was supported in part by NSF awards CCF-2422926 and CCF-2608359. S.D.G. and S.L. were supported in part by the PNRR MUR project IR0000013-SoBigData.it project, by the MUR PRIN grant 2022EKNE5K
(Learning in Markets and Society) and by the FAIR (Future
Artificial Intelligence Research) project PE0000013, funded by the NextGenerationEU program within the PNRR-
PE-AI scheme (M4C2, investment 1.3, line on Artificial Intelligence). S.D.G.\ was also supported in part by the Institute for
Complex Systems (Italian National Research Council).

\bibliographystyle{plainnat}
\bibliography{references}

\clearpage
\renewcommand{\theHtheorem}{appendix.\Alph{section}.\arabic{theorem}}
\renewcommand{\theHlemma}{appendix.\Alph{section}.\arabic{theorem}}
\renewcommand{\theHdefinition}{appendix.\Alph{section}.\arabic{theorem}}
\renewcommand{\theHproposition}{appendix.\Alph{section}.\arabic{theorem}}
{\noindent \LARGE  \bf Appendix}

\appendix
\crefalias{section}{appendix}
\section{Lower Bounds}\label{sec:lb}

In this appendix, we recall and (slightly) generalize a result of (see \citet[Section 3]{RCCFHKL24}), which shows that the regret rates from \Cref{thm:noise-oco} are tight in $T, k$ and $\delta$, up to logarithmic factors in $T$. Their result can indeed be extended to work for linear (and, thus, general convex) functions.
% \footnote{The lower bound proved in \citep[Section 3]{RCCFHKL24}) works even in the stochastic setting, while the one stated here is technically weaker since it holds for the stronger adversarial setting.} 
In particular, these lower bounds
hold for a learner that 
has access to the more powerful \emph{global probes}, i.e., it observes the global minimum of the function at that round, and the probe is noiseless. Before stating and proving the result, we stress that it shows complete tightness (also in $d$) of the rate in \Cref{thm:experts_regret}, since the proof below reduces to a lower bound on the problem of prediction with expert advice on the vertices of $\Delta^{d-1}$.

\begin{theorem}
\label{thm:thmlbcvx}
    There exists a family of instances defined as linear functions bounded in $[0,1]$ over the $(d-1)$-dimensional simplex $\Delta^{d-1}$ such that for $k \le O\left(\tfrac{T\ln^{3/2} d}{d}\right)$ and $T\geq O\left(\tfrac{d^2}{\ln^2d}\right)$:\footnote{We require this to satisfy the sufficient condition in \Cref{eq:BE-condition} needed to apply \Cref{lem:binom-min}: in particular, $T\geq \nicefrac{200d^2}{\ln d}$ and $k \le \nicefrac{T\ln^{3/2} d}{700d}$ suffices.}
    \begin{itemize}
        \item[(1)] Any algorithm equipped with $k$ global probes must suffer regret
        \[
            \E{}{\Reg_T(\ALG)} \geq \Omega\left(\min\left\{\sqrt{T\ln d}, \frac{T\ln d}{k}\right\}\right);
        \]
        \item[(2)] Any algorithm equipped with $k$ pairwise comparison $\delta$-noisy probes must suffer regret
        \[
            \E{}{\Reg_T(\ALG)} \geq \Omega\left(\min\left\{\sqrt{T\ln d}, \frac{T\ln d}{k|1-2\delta|}\right\}\right).
        \]
    \end{itemize}
\end{theorem}
\begin{proof}
    The proof uses Yao's Minimax principle, so we need to prove a hardness result for a deterministic algorithm against a randomized sequence of losses. The family of instances is described as the following linear function $f_t(x) = f_t^\top x$ defined on the $(d-1)$-dimensional simplex $\Delta^{d-1}$: each coordinate of the vector $f_t$ is $1$ with probability $\xi$ independently from all other coordinates and across rounds.

    Let us first observe that $\min_{x \in \Delta^{d-1}} \sum_{t=1}^T f_t^\top x \le \min_{x \in \{e_1, \ldots, e_d\}} \sum_{t=1}^T f_t^\top x$ since the vertices are part of the simplex. In addition, since $\sum_{i=1}^d x_i = 1$ for all $x \in \Delta^{d-1}$ and the vertices $\{e_i\}_{i=1}^d$ form the canonical basis in $\R^d$, we have
    \begin{align*}
        \min_{x \in \Delta^{d-1}} \sum_{t=1}^T f_t^\top x &= \min_{x \in \Delta^{d-1}} \sum_{t=1}^T \sum_{i=1}^d f_{ti}x_i = \min_{x \in \Delta^{d-1}} \sum_{i=1}^d  \left(\sum_{t=1}^T f_{ti}\right) x_i \\
        &\ge \min_{x \in \Delta^{d-1}} \sum_{i=1}^d  \left(\min_{j=1}^d\sum_{t=1}^T f_{tj}\right) x_i = \min_{j=1}^d\sum_{t=1}^T f_{tj} = \min_{x \in \{e_1, \ldots, e_d\}} \sum_{t=1}^T f_t^\top x.
    \end{align*}
    Therefore, $\min_{x \in \Delta^{d-1}} \sum_{t=1}^T f_t^\top x = \min_{x \in \{e_1, \ldots, e_d\}} \sum_{t=1}^T f_t^\top x$ and we can restrict our attention to vertices $\{e_1, \ldots, e_d\}$ of the simplex, as the minimal cumulative loss is achieved at one of them. 
    
    By the above construction, for each $x \in \{e_1, \ldots, e_d\}$, it holds that $\sum_{t=1}^T f_t^\top x \sim \mathrm{Bin}(T, \xi)$, which means that $\E{}{\sum_{t=1}^T f_t^\top x} = T\xi$. Hence, the expected minimal loss among vertices is, by \Cref{lem:binom-min} (below) with $d$ sufficiently large (so that $2\ln d - 3\ln\ln d \ge \ln d$):
    \[
        \E{}{\min_{x \in \{e_1, \ldots, e_d\}} \sum_{t=1}^T f_t^\top x} \leq T\xi - \frac{1}{25}\sqrt{T\xi(1-\xi)\ln d}
    \]
    
    Furthermore, by linearity and independence of $f_t$ from the past, we have that for all $x_t\in\Delta^{d-1}$:
    \[
        \E{}{f_t^\top x_t}
        =
        \E{}{\sum_{i=1}^d f_{ti}x_{ti}}
        =
        \sum_{i=1}^d \E{}{f_{ti}}x_{ti}
        = 
        \xi \cdot \sum_{i=1}^d x_{ti}
        =
        \xi,
    \]
    since $\sum_{i=1}^d x_{ti}=1$.
    Therefore the expected cumulative loss of the algorithm on the $T-k$ non-probing rounds is exactly $(T-k)\xi$.

%---------------------------------------------------
    
    We can now proceed with the proof of (1):
    the algorithm's expected loss during probing rounds is $k\xi^d$ because we need all vertices to have loss $1$ for the minimum to be $1$. Thus, the overall expected loss of the algorithm reads
    \[
        (T-k)\xi + k\xi^d = T\xi -k\xi(1-\xi^{d-1}).
    \]
    The regret reads
    \begin{align*}
        \E{}{\Reg_T(\ALG)} &\geq \frac{1}{25}\left(\sqrt{T\xi(1-\xi)\ln d} - 25k\xi(1-\xi^{d-1})\right) \\
        &\geq \frac{1}{5\cdot 10^3}\min\left\{\sqrt{T\ln d}, \frac{T\ln d}{k}\right\},
    \end{align*}
    where we have chosen $\xi=\frac{1}{2}$ for $k \leq \frac{1}{50}\sqrt{T\ln d}$ and $\xi = \frac{T\ln d}{5\cdot 10^3k^2} < \frac{1}{2}$ otherwise.

We conclude with the proof of (2): at any time step, let us condition on
the past randomness of the loss functions and let us assume that, given this conditioning, the deterministic learner chooses to probe: fix the two queried points.
Since the current loss vector is independent of the past, its coordinates are still i.i.d.\ $\mathrm{Bernoulli}(\xi)$. Now, the pairwise
probe outcome is a single noisy binary signal about the current loss
vector and a $\mathrm{Bernoulli}(\nicefrac{1}{2})$ random bit handling tie breaking. Moreover, since the learner is allowed to know $\delta$, if
$\delta>\nicefrac12$ it may invert the probe outcome. Thus the effective
probability that this binary signal is incorrect is $\min\{\delta,1-\delta\}$.

By \Cref{lem:one-bit-bayes-risk} applied with $\rho = \min\{\delta,1-\delta\}$, even if the learner is allowed to play
an arbitrary point of $\Delta^{d-1}$ after observing this signal, its
expected loss on the probed round is at least
\[
    2\min\{\delta,1-\delta\} \cdot \xi+(1-2\min\{\delta,1-\delta\}) \cdot \xi^2
    =
    \xi-|1-2\delta|\xi(1-\xi).
\]
Therefore, if the learner uses $k'\le k$ probes, then
\[
    \E{}{L_T(\ALG)}
    \ge
    (T-k')\xi
    +
    k'\left(\xi-|1-2\delta|\xi(1-\xi)\right)
    \ge
    T\xi-k|1-2\delta|\xi(1-\xi).
\]
Consequently,
\[
    \E{}{\Reg_T(\ALG)}
    \ge
    \frac{1}{25}
    \left(
        \sqrt{T\xi(1-\xi)\ln d}
        -
        25k|1-2\delta|\xi(1-\xi)
    \right).
\]
We split into two cases: first, if $|1-2\delta| \le \nicefrac{\sqrt{T\ln d}}{50k}$, we choose $\xi=\nicefrac12$, and then
\[
    \E{}{\Reg_T(\ALG)}
    \ge
    \frac{1}{25}
    \left(
        \frac12\sqrt{T\ln d}
        -
        \frac{25}{4}k|1-2\delta|
    \right)
    \ge
    \frac{3}{200}\sqrt{T\ln d}.
\]
Otherwise, $|1-2\delta| > \nicefrac{\sqrt{T\ln d}}{50k}$, and we choose $\xi\in(0,\nicefrac12]$ such that
\[
    \xi(1-\xi)
    =
    \frac{T\ln d}{10^4 k^2|1-2\delta|^2},
\]
which is feasible because the case assumption implies the right-hand side
is at most $\nicefrac14$. With this choice,
\begin{align*}
    \E{}{\Reg_T(\ALG)}
    \ge
    \frac{1}{25}
    \left(
        \frac{T\ln d}{100k|1-2\delta|}
        -
        \frac{25T\ln d}{10^4k|1-2\delta|}
    \right)  \ge
    \frac{T\ln d}{4\cdot 10^3 k|1-2\delta|}.
\end{align*}
Combining the two cases yields
\[
    \E{}{\Reg_T(\ALG)}
    \ge
        \min\left\{
            \frac{3}{200}\sqrt{T\ln d},
            \frac{T\ln d}{4\cdot 10^3k|1-2\delta|}
        \right\},
\]
with the convention that the second term is $\infty$ when
$\delta=\nicefrac12$.
\end{proof}

As promised in the earlier proof, we are left to show that one noisy binary signal cannot reduce Bernoulli loss too much. We then also have to bound the expected minimum of i.i.d. binomials. 

We begin with the first claim: the uniform tie breaking convention after \Cref{def:probes_oco} ensures that the probe outcome is only
a noisy binary signal of the comparison and a $\mathrm{Bernoulli}(\nicefrac{1}{2})$ random bit which settles ties.

\begin{lemma}
\label{lem:one-bit-bayes-risk}
Let $A=(A_1,\ldots,A_d)$, where the coordinates are independent
$\mathrm{Bernoulli}(\xi)$ random variables, and let $U\sim \mathrm{Bernoulli}(\nicefrac{1}{2})$, independently of $A$. Let $S\in\{0,1\}$ be any
binary random variable measurable with respect to $(A, U)$. Consider $B\in\{0,1\}$
to be a noisy version of $S$ defined as $S$ with probability $1-\rho$ and $1-S$ otherwise, where the random choice above is independent of $(A, U)$, and where
$\rho\in[0,\nicefrac{1}{2}]$. Letting $z_B=f(B)$, for a function $f:\{0, 1\}\rightarrow\Delta^{d-1}$, it holds that:
\[
    \E{}{A^\top z_B}
    \ge
    2\rho\xi+(1-2\rho)\xi^2.
\]
\end{lemma}
\begin{proof}
As for the proof of \Cref{thm:thmlbcvx}, we have that since the function $f$ maps to the simplex $\Delta^{d-1}$, it suffices to prove the lower bound for $f$ mapping to $\{e_i\}_{i=1}^d \subseteq \Delta^{d-1}$. Let $f(0) = z_0$ and $f(1) = z_1$, for $z_0, z_1\in \{e_i\}_{i=1}^d$. We want to lower bound the following: 
\begin{align}
\E{}{A^\top z_B} &= \E{}{\ind{B=0}A^\top  z_0 + \ind{B=1}A^\top  z_1} \notag\\
& = (1-\rho)\E{}{\ind{S=0}A^\top z_0 + \ind{S=1}A^\top z_1} \notag \\
& \quad + \rho\E{}{\ind{S=0}A^\top z_1 + \ind{S=1}A^\top z_0}\label{eq:pivot_a2}
\end{align}
Disregarding the convex combination, summing the two expectations above gives $\E{}{A^\top (z_0+z_1)} = 2\xi$, since $A$ is a Bernoulli vector and $z_0, z_1\in \{e_i\}_{i=1}^d$. This means that: $$\E{}{\ind{S=0}A^\top z_1 + \ind{S=1}A^\top z_0} = 2\xi-\E{}{\ind{S=0}A^\top z_0 + \ind{S=1}A^\top z_1}.$$
Letting $\zeta = \E{}{\ind{S=0}A^\top z_0 + \ind{S=1}A^\top z_1}$ and plugging the above in \Cref{eq:pivot_a2} we get: $\E{}{A^\top z_B} = (1-\rho)\zeta + \rho(2\xi-\zeta) = 2\rho\xi+\zeta(1-2\rho)$.\footnote{On the event that the coordinates corresponding to $z_0$ and $z_1$ are both equal to $1$, the quantity inside the expectation defining $\zeta$ equals $1$ regardless of whether $S=0$ or $S=1$. Thus the argument is insensitive to how ties are broken; it only uses that $S$ is a binary random variable measurable with respect to $(A,U)$.} To prove the claim, it thus suffices to prove $\zeta \geq \xi^2$.

The first case is $z_0=z_1$, and the lower bound follows simply because $\zeta = \E{}{A^\top  z_0} = \xi\geq \xi^2$. Now suppose that $z_0\neq z_1$: the event $\{A^\top z_0=1 \}\cap \{A^\top z_1 = 1\}$ has probability $\xi^2$, since $A$ is a vector of independent Bernoulli variables and the event happens if and only if the entries in $A$ corresponding to the two non-null entries in $z_0$ and $z_1$ are non-null. Therefore,
\begin{align*}
\zeta &= \E{}{\ind{S=0}A^\top z_0 + \ind{S=1}A^\top z_1} \\
&\geq \E{}{\ind{\{A^\top z_0=1 \}\cap \{A^\top z_1 = 1\}}\left(\ind{S=0}A^\top z_0 + \ind{S=1}A^\top z_1\right)} \\
& = \P{}{A^\top z_0=1 \wedge A^\top z_1 = 1} = \xi^2,
\end{align*}
as desired.
\end{proof}

We return to bounding the expected minimum of i.i.d. binomials from above. The following result is arguably folklore, but in the absence of a convenient source we provide a proof for completeness. 

\begin{lemma}\label{lem:binom-min}
Let $d\ge 2$ be an integer, let $\xi\in(0,1)$ and $T\in\mathbb N$, and let
$X_1,\dots,X_d$ be i.i.d.\ $\mathrm{Bin}(T,\xi)$.
Define
\[
\mu = \E{}{X_1}=T\xi,
\qquad
\sigma^2 = \V{}{X_1}=T\xi(1-\xi),
\qquad
\psi(s)=\P{}{X_1\le \mu-s}\ \ \ (s\ge0).
\]
Let
\[
u_d = \sqrt{\,2\ln d - 3\ln\ln d\,},
\qquad
s_d = \sigma\,u_d.
\]
Assume the following:
\begin{equation}\label{eq:BE-condition}
\frac{0.4690}{\sqrt{T\xi(1-\xi)}} \ \le\ \frac{\ln d}{8\sqrt{\pi}\,d}.
\end{equation}
Then
\begin{equation}\label{eq:main-bound}
\E{}{\min_{1\le i\le d} X_i}
\ \le\
T\xi\;-\; s_d\Bigl(1-d^{-1/(8\sqrt{\pi})}\Bigr).
\end{equation}
In particular, since $d\geq 2$, we also have the weaker bound:
\[
\E{}{\min_{1\le i\le d} X_i}
\ \le\
T\xi\;- \frac{1}{25}\sqrt{T\xi(1-\xi)}\;\sqrt{\,2\ln d - 3\ln\ln d\,}.
\]
\end{lemma}

\begin{proof}
Let us first define
\[
M = \max_{1\le i\le d}(\mu-X_i)\ \ge\ 0,
\]
so that then $\min_i X_i = \mu - M$, and hence
\begin{equation}\label{eq:expect-min-mu-M}
\E{}{\min_{1\le i\le d}X_i} = \mu - \E{}{M}.
\end{equation}
The proof proceeds in steps:\\

\noindent \textbf{Step 1.} We give a general lower bound on $\E{}{M}$ in terms of $\psi$: for any $s\ge0$, by independence,
\[
\P{}{M<s}
= \P{}{\mu-X_i<s\ \forall i}
= \prod_{i=1}^d \P{}{\mu-X_i<s}
= (1-\psi(s))^d.
\]
Therefore
\begin{equation}\label{eq:tail-M}
\P{}{M\ge s}=1-(1-\psi(s))^d.
\end{equation}
Using $\E{}{M}=\int_0^\infty \P{}{M\ge t}\,dt$, we obtain for any $s\ge0$,
\begin{equation}\label{eq:EM-lower}
\E{}{M}\ \ge \int_0^s \P{}{M\geq t}dt\ \ge\ \int_0^{s}\P{}{M\ge s}\,dt
\ =\ s\,\P{}{M\ge s}
\ =\ s\bigl(1-(1-\psi(s))^d\bigr).
\end{equation}
Combining \Cref{eq:expect-min-mu-M,eq:EM-lower} gives, for any $s\ge0$,
\begin{equation}\label{eq:min-upper-general}
\E{}{\min_{1\le i\le d}X_i}
\ \le\
\mu \;-\; s\bigl(1-(1-\psi(s))^d\bigr).
\end{equation}

\noindent \textbf{Step 2.} Next, we lower-bound $\psi(s_d)$ via Berry--Esseen and a Mills-ratio bound: let us expand $X_1=\sum_{j=1}^T B_j$ where $B_j\sim\mathrm{Bern}(\xi)$ i.i.d..
Let $Y_j=B_j-\xi$, so that $S_T = \sum_{j=1}^T Y_j = X_1-\mu$.

% and note $\V{}{S_T}=T\xi(1-\xi)=\sigma^2$.

\smallskip
By the classical Berry--Esseen inequality with the upper bound
$C_0<0.4690$ for i.i.d.\ summands with finite third moment \citep{shevtsova2013absolute, Shevtsova14}, one has
\begin{equation}\label{eq:BE}
\sup_{x\in\R}~\left|\P{}{\frac{S_T}{\sigma}\le x}-\Phi(x)\right|
\ \le\
0.4690\cdot \frac{\E{}{|Y_1|^3}}{(\E{}{Y_1^2})^{3/2}}\cdot \frac{1}{\sqrt{T}},
\end{equation}
where $\Phi$ is the standard normal CDF.
For centered Bernoulli variables, $\E{}{Y_1^2}=\xi(1-\xi)$ and
\[
\E{}{|Y_1|^3} = \xi(1-\xi)^3 + (1-\xi)\xi^3 = \xi(1-\xi)\bigl(\xi^2+(1-\xi)^2\bigr)
\ \le\ \xi(1-\xi).
\]
Hence $\E{}{|Y_1|^3}/(\E{}{Y_1^2})^{3/2} \le 1/\sqrt{\xi(1-\xi)}$, and \Cref{eq:BE} yields
\begin{equation}\label{eq:BE-simplified}
\sup_{x\in\R}~\left|\P{}{\frac{S_T}{\sigma}\le x}-\Phi(x)\right|
\ \le\
\frac{0.4690}{\sqrt{T\xi(1-\xi)}}.
\end{equation}
Also notice that
\[
\psi(s_d)=\P{}{X_1\le \mu-s_d}
=\P{}{\frac{S_T}{\sigma}\le -u_d}, 
\]
since $s_d=\sigma u_d$ by definition.
By \Cref{eq:BE-simplified} we then have:
\begin{equation}\label{eq:psi-lower-1}
\psi(s_d)\ \ge\ \Phi(-u_d)\;-\;\frac{0.4690}{\sqrt{T\xi(1-\xi)}}.
\end{equation}

We next lower-bound $\Phi(-u_d)$. To do this, we use the Mills-ratio lower bound (see \citet{Gordon41}):
for all $u>0$,
\begin{equation}\label{eq:mills}
\Phi(-u)\ \ge\ \frac{1}{\sqrt{2\pi}}\cdot \frac{u}{1+u^2}\,e^{-u^2/2}.
\end{equation}
Since $d\ge 2$, we have
$u_d^2=2\ln d-3\ln\ln d > 1$, so $u_d\ge 1$.
Therefore $1+u_d^2 \le 2u_d^2$, hence $\frac{u_d}{1+u_d^2}\ge \frac{1}{2u_d}$, and
\Cref{eq:mills} gives
\begin{equation}\label{eq:Phi-lower-2}
\Phi(-u_d)
\ \ge\
\frac{1}{\sqrt{2\pi}}\cdot \frac{1}{2u_d}\,e^{-u_d^2/2}.
\end{equation}
Moreover,
\[
e^{-u_d^2/2} = e^{-(2\ln d - 3\ln\ln d)/2} = \frac{(\ln d)^{3/2}}{d},
\]
and also $u_d^2 \le 2\ln d$ implies $u_d\le \sqrt{2\ln d}$, i.e.\ $\frac{1}{u_d}\ge \frac{1}{\sqrt{2\ln d}}$.
Plugging these into \Cref{eq:Phi-lower-2},
\[
\Phi(-u_d)
\ \ge\
\frac{1}{\sqrt{2\pi}}\cdot \frac{1}{2}\cdot \frac{1}{\sqrt{2\ln d}}\cdot \frac{(\ln d)^{3/2}}{d}
=
\frac{\ln d}{4\sqrt{\pi}\,d}.
\]
Insert this into \Cref{eq:psi-lower-1} and use the assumption in \Cref{eq:BE-condition} to obtain
\begin{equation}\label{eq:psi-lower-final}
\psi(s_d)\ \ge\ \frac{\ln d}{4\sqrt{\pi}\,d}\;-\;\frac{\ln d}{8\sqrt{\pi}\,d}
\ =\ \frac{\ln d}{8\sqrt{\pi}\,d}.
\end{equation}

\noindent \textbf{Step 3.}
We apply \Cref{eq:min-upper-general} with $s=s_d$.
Using $(1-x)^d \le e^{-dx}$ for $x\in[0,1]$, we get
\[
1-(1-\psi(s_d))^d \ \ge\ 1-e^{-d\psi(s_d)}.
\]
With \Cref{eq:psi-lower-final}, $d\psi(s_d)\ge \frac{\ln d}{8\sqrt{\pi}}$, hence
\[
1-e^{-d\psi(s_d)}
\ \ge\
1-\exp{-\frac{\ln d}{8\sqrt{\pi}}}
=
1-d^{-1/(8\sqrt{\pi})}.
\]
Substituting into \Cref{eq:min-upper-general} yields
\[
\E{}{\min_{1\le i\le d}X_i}
\ \le\ \mu - s_d\Bigl(1-d^{-1/(8\sqrt{\pi})}\Bigr),
\]
which is exactly \Cref{eq:main-bound}.
\end{proof}

\section{A Strawman Algorithm for Experts: a Uniform and a Hedge Probe}
\label{app:strawman-experts}

In this appendix, we focus on the case of experts. We show how a uniform probe together with a Hedge probe (\Cref{alg:hedge-with-uniform}) can yield non-trivial regret rates in the noiseless case, albeit losing an additional factor $d$. This algorithm is a simple adaptation of Hedge \citep{FreundS97} (subsequently denoted as $\ALG_0$), which, during probing rounds $Q$, probes one expert uniformly at random and one expert according to the Hedge distribution $p_t$ at round $t \in Q$. We set the learning rate to $\eta = \max\left(\sqrt{\nicefrac{\ln d}{T}}, \nicefrac{k}{dT+k}\right)$.

\ifcoltversion
\begin{algorithm2e}[b!]
\SetAlgoNoEnd
    \caption{Uniform Probe Hedge}
    \label{alg:hedge-with-uniform}
    % \begin{algorithmic}[1]
    \SetKwInOut{Input}{Input}
    
        \textbf{Input:} Sequence of gradient vectors $f_t$, probe budget $k \le T$
        
        Sample $k$ out of $T$ rounds uniformly at random and denote this random set by $Q$
        
        Set $\eta = \max\left(\sqrt{\frac{\ln d}{T}}, \frac{k}{dT+k}\right)$ \textbf{if} $T\geq 4\ln d$ \textbf{else} $\eta = \frac{k}{dT+k}$
        
        Initialize $w_1(x) = 1$ for all $x \in \{e_1, \ldots, e_d\}$
        
        \For{$t \in \{1, \ldots, T\}$}{
            
            Let $W_t = \sum_{x \in \{e_1, \ldots, e_d\}} w_t(x)$ and $p_t(x) = \frac{w_t(x)}{W_t}$
            
            \eIf{$t \in Q$}{ 
                Probe $y_1 \sim p_t$ and $y_2 \sim \text{Unif}(\{e_1, \ldots, e_d\})$
                
                Play point $y^\star_t = \arg\min\{f_t^\top y_1, f_t^\top y_2\}$
            }{
                Select $x \sim p_t$
            }
            
            Observe $f_t^\top x ~\forall\, x \in \{e_1, \ldots, e_d\}$
            
            Update $w_{t+1}(x) = w_t(x) \cdot \exp{-\eta f_t^\top x} \ \forall\, x \in \{e_1, \ldots, e_d\}$
        }
    % \end{algorithmic}
\end{algorithm2e}
\else
\begin{algorithm}[t!]
\SetAlgoNoEnd
    \caption{Uniform Probe Hedge}
    \label{alg:hedge-with-uniform}
    % \begin{algorithmic}[1]
    \SetKwInOut{Input}{Input}
    
        \textbf{Input:} Sequence of gradient vectors $f_t$, probe budget $k \le T$
        
        Sample $k$ out of $T$ rounds uniformly at random and denote this random set by $Q$
        
        Set $\eta = \max\left(\sqrt{\frac{\ln d}{T}}, \frac{k}{dT+k}\right)$
        
        Initialize $w_1(x) = 1$ for all $x \in \{e_1, \ldots, e_d\}$
        
        \For{$t \in \{1, \ldots, T\}$}{
            
            Let $W_t = \sum_{x \in \{e_1, \ldots, e_d\}} w_t(x)$ and $p_t(x) = \frac{w_t(x)}{W_t}$
            
            \eIf{$t \in Q$}{ 
                Probe $y_1 \sim p_t$ and $y_2 \sim \text{Unif}(\{e_1, \ldots, e_d\})$
                
                Play point $y^\star_t = \arg\min\{f_t^\top y_1, f_t^\top y_2\}$
            }{
                Select $x \sim p_t$
            }
            
            Observe $f_t^\top x ~\forall\, x \in \{e_1, \ldots, e_d\}$
            
            Update $w_{t+1}(x) = w_t(x) \cdot \exp{-\eta f_t^\top x} \ \forall\, x \in \{e_1, \ldots, e_d\}$
        }
\end{algorithm}
\fi

Before stating the theorem, we first recall that, in prediction with expert advice, instances are described by linear functions $f_t(x) = f_t^\top x$, for some vector $f_t$, where $x \in \{e_1, \ldots, e_d\}$ is one of the vertices of the $(d-1)$-dimensional simplex $\Delta^{d-1}$, i.e., $K = \{e_1, \ldots, e_d\}$: coordinate $i$ of vector $f_t$ therefore represents the loss of expert $e_i$.

\begin{theorem}\label{thm:perfect-subopt}
    Consider the problem of prediction with expert advice with full feedback and $k$ pairwise comparison probes, with domain $K = \{e_1, \ldots, e_d\}$.
    Then, for any sequence of functions $\{f_t\}_{t=1}^T$ in $[0, 1]$, \Cref{alg:hedge-with-uniform}, denoted as $\ALG_k$, guarantees regret
    \[
        \E{}{\Reg_T(\ALG_k)} \leq O\left(\min\left\{\sqrt{T\ln d}, \frac{T d\ln d}{k}\right\}\right).
    \]
\end{theorem}

To prove the theorem, we state a useful lemma from \citet{RCCFHKL24}, which bounds the regret of the probeless version $\ALG_0$ of $\ALG_k$; it directly follows from Lemma 2.1 in \citet{RCCFHKL24} shifting losses, denoting $x^\star_t = \arg\min_{x\in \{e_i\}_{i=1}^d} f_t(x)$:
\begin{lemma}
    For the problem of prediction with expert advice with full feedback and domain $K = \{e_1, \ldots, e_d\}$, it holds that the expected regret of $\ALG_0$ with $\eta < 1$ is bounded above as
    \begin{align}\label{eq:start-subopt}
        \E{}{\Reg_T(\ALG_0)} \le \frac{\ln d}{\eta(1-\eta)} + \frac{\eta}{1-\eta} \sum_{t = 1}^T f_t^\top (x^\star - x^\star_t).
    \end{align}
\end{lemma}
\ifcoltversion
\begin{proof}[of \Cref{thm:perfect-subopt}]
\else
\begin{proof}[Proof of \Cref{thm:perfect-subopt}]
\fi
    First, we relate the regret of $\ALG_k$ to that of $\ALG_0$. To this end, 
    observe that, since $\ALG_k$ is provided with full feedback, then, the distribution $p_t$ Hedge keeps over experts is not affected by earlier probes or earlier decisions taken by the algorithm. Let $z_t$ denote the point/expert played by the algorithm in (probing or non-probing) round $t$. During non-probing rounds $t \in \overline{Q} =\{1,\dots,T\}\setminus Q$, the expected loss $\ALG_k$ suffers is 
    \[
        \E{}{f_t^\top z_t \mid t \in \overline{Q}} = \sum_{x \in \{e_1, \ldots, e_d\}} p_t(x) \cdot f_t^\top x.
    \]
    Moreover, in rounds $t \in Q$, $\ALG_k$ probes the expert of minimum loss with probability $\nicefrac{1}{d}$. In the event that the expert of minimum loss is not probed, we have that, in expectation, the algorithm $\ALG_k$ suffers at most the loss $\ALG_0$ (Hedge) would have suffered. Hence, the expected instantaneous loss of $\ALG_k$ at round $t \in Q$ is
    \begin{align*}
        \E{}{f_t^\top z_t \mid t \in Q} \leq \left(1 - \frac{1}{d}\right) \sum_{x \in \{e_1, \ldots, e_d\}} p_t(x) \cdot f_t^\top x + \frac{f_t^\top x^\star_t}{d}.
    \end{align*}
    Summing over all $t$, by linearity of expectation, we have
    that the expected loss of the algorithm is:
    \begin{align*}
        \E{}{L_T(\ALG_k)}&= \E{}{\sum_{t=1}^T f_t^\top z_t} = \sum_{t=1}^T \E{}{f_t^\top z_t \cdot \ind{t \in \overline{Q}}} + \sum_{t=1}^T \E{}{f_t^\top z_t \cdot \ind{t \in Q}}\\
        &= \sum_{t=1}^T \E{}{f_t^\top z_t \mid t \in \overline{Q}} \cdot \P{}{t \in \overline{Q}} + \sum_{t=1}^T \E{}{f_t^\top z_t \mid t \in Q} \cdot \P{}{t \in Q}\\
        &\le \left(1-\frac{k}{T}\right) \sum_{t=1}^T\sum_{x \in \{e_1, \ldots, e_d\}} p_t(x) \cdot f_t^\top x \\
        &\qquad\qquad + \frac{k}{T} \sum_{t=1}^T\left(\left(1 - \frac{1}{d}\right) \sum_{x \in \{e_1, \ldots, e_d\}} p_t(x) \cdot f_t^\top x + \frac{f_t^\top x^\star_t}{d}\right)\\
        &= \left(1-\frac{k}{dT}\right)\E{}{L_T(\ALG_0)} + \frac{k}{dT} \sum_{t=1}^T f_t^\top x^\star_t.
    \end{align*}
    Combining the above with \Cref{eq:start-subopt}, we have
    \begin{align*}
        \E{}{\Reg_T(\ALG_k)} &\le \left(1-\frac{k}{dT}\right)\E{}{\Reg_T(\ALG_0)} - \frac{k}{dT} \sum_{t=1}^T f_t^\top (x^\star - x^\star_t)\\
        &\le \left(1-\frac{k}{dT}\right)\frac{\ln d}{\eta(1-\eta)} + \left(\frac{\eta}{1-\eta} - \frac{k}{dT}\right) \sum_{t=1}^T f_t^\top (x^\star - x^\star_t)\\
        &\le \min\left\{4\sqrt{T\ln d}, \frac{2Td\ln d}{k}\right\},
    \end{align*}
    by assuming $T\geq 4\ln d$ and choosing $\eta = \max\left(\sqrt{\frac{\ln d}{T}}, \frac{k}{dT + k}\right)$, since $\sum_{t=1}^T f_t^\top (x^\star - x^\star_t) \leq T$. If $T < 4\ln d$, the regret guarantee is satisfied since $\eta = \frac{k}{dT+k}$ and the regret is bounded by $T$.
\end{proof}

\end{document}